\documentclass{article}

\usepackage{amsmath,amssymb,amsthm}
\usepackage{braket}
\usepackage{mathrsfs}
\usepackage{stmaryrd}
\usepackage{subcaption}
\usepackage{fullpage}
\usepackage{hyperref}
\usepackage{xcolor}
\usepackage{authblk}

\newtheorem{dfn}{Definition}

\newtheorem {exa}{Example}

\newtheorem{theorem}{Theorem}

\newcommand{\mc}{\mathcal}

\newcommand{\freepreg}[1]{\mathsf{Preg}_{#1}}
\newcommand{\lang}[1]{\ensuremath{\textit{#1}}}
\newcommand{\FHilb}{\mathbf{FHilb}}
\newcommand{\functor}[1]{\mathsf{#1}}

\newcommand{\CPM}[1]{\mathbf{CPM}(#1)}
\newcommand{\CPMC}{\CPM{\mathcal{C}}}
\newcommand{\cpmsem}{\functor{S}}
\newcommand{\cpmpure}{\functor{E}}
\newcommand{\sem}[1]{\rho( #1 )}
\newcommand{\define}[1]{{\bf #1}}

\usepackage{qtree}
\usepackage{tikz}
\usetikzlibrary{decorations.markings}
\usetikzlibrary{positioning, arrows, shapes}

\pgfdeclarelayer{edgelayer}
\pgfdeclarelayer{nodelayer}
\pgfsetlayers{edgelayer,nodelayer,main}

\tikzstyle{none}=[inner sep=0pt]
\tikzstyle{small circ}=[inner sep=0pt, circle, minimum size = 0.2cm,fill=white,draw=black]
\tikzstyle{small_node}=[inner sep=0pt, circle, minimum size = 0.2cm,fill=white,draw=black]
\tikzstyle{blank}=[inner sep=0pt, circle,fill=white,draw=white]
\tikzstyle{box}=[rectangle, minimum size = 0.5cm,fill=white,draw=black]

\definecolor{amber}{rgb}{1.0, 0.49, 0.0}

\usepackage[textsize=footnotesize]{todonotes}

\definecolor{amaranth}{rgb}{0.9, 0.17, 0.31}

\usepackage[capitalise]{cleveref}

\newcommand{\tr}{\textrm{tr}}

\newcommand{\PSD}{\mathrm{PSD}}
\newcommand{\M}{\mathcal{M}}
\newcommand{\be}{\begin{eqnarray}}
\newcommand{\ee}{\end{eqnarray}}
\newcommand{\bi}{\begin{itemize}}
\newcommand{\ei}{\end{itemize}}
\newtheorem{lemma}[theorem]{Lemma}

\newcommand{\diag}{\mathrm{diag}}

\title{Cats climb entails mammals move: preserving hyponymy in compositional distributional semantics}
\author[1]{Gemma De las Cuevas}
\author[1]{Andreas Klingler}
\author[2]{Martha Lewis}
\author[3]{Tim Netzer}
\affil[1]{\small{Institute for Theoretical Physics, Technikerstr.\ 21a, A-6020 Innsbruck, Austria}}
\affil[2]{ILLC, University of Amsterdam, Amsterdam, The Netherlands}
\affil[3]{Department of Mathematics, Technikerstr.\ 13, A-6020 Innsbruck, Austria}
\date{\small\today}

\begin{document}
\maketitle

\begin{abstract}
To give vector-based representations of meaning more structure, one approach is to use positive semidefinite (psd) matrices. These allow us to model similarity of words as well as the \emph{hyponymy} or \emph{is-a} relationship. Psd matrices can be learnt relatively easily in a given vector space $M\otimes M*$, but to compose words to form phrases and sentences, we need representations in larger spaces. In this paper, we introduce a generic way of composing the psd matrices corresponding to words. 
We propose that psd matrices for verbs, adjectives, and other functional words be lifted to completely positive (CP) maps that match their grammatical type. 
This lifting is carried out by our composition rule called  Compression, \textsf{Compr}. 
In contrast to previous composition rules like \textsf{Fuzz} and \textsf{Phaser} (a.k.a.\  \textsf{KMult} and \textsf{BMult}), \textsf{Compr} preserves hyponymy. 
Mathematically, \textsf{Compr} is itself a CP map, 
and is therefore linear and generally non-commutative. 
We give a number of proposals for the structure of \textsf{Compr}, based on spiders, cups and caps, and generate a range of composition rules. 
We test these rules on a small sentence entailment dataset, and see some improvements over the performance of \textsf{Fuzz} and \textsf{Phaser}.
\end{abstract}

\section{Introduction}
\label{sec:intro}

Vector-based representations of words, with similarity measured by the inner product of the normalised word vectors, have been extremely successful in a number of applications. However, as well as similarity, there are a number of other important relations between words or concepts, one of these being \emph{hyponymy} or the \emph{is-a} relation. Examples of this are that \emph{cat} is a hyponym of \emph{mammal}, but we can also apply this to verbs, and say that \emph{sprint} is a hyponym of \emph{run}. Within standard vector-based semantics based on  co-occurrence statistics, there is no standard way of representing hyponymy between word vectors. There have been a number of alternative approaches to building word vectors that can represent these relationships, but most of these operate at the single word level.
Of course, words can be composed to form phrases and sentences, and we use a variant of the categorical compositional distributional (DisCoCat) approach introduced in \cite{Co10c}. This approach uses a category-theoretic stance. It models syntax in one category, call it the \emph{grammar category}, and semantics in another, call it the \emph{meaning category}. A functor from the grammar category to the meaning category is defined, so that the grammatical reductions on the syntactic side can be translated into morphisms on the meaning side. The standard instantiation models meaning within the category of vector spaces and linear transformations, so that nouns are represented as vectors, and \emph{functional words} such as verbs and adjectives are represented as multilinear maps, or alternatively, matrices and tensors.

Within DisCoCat, choices can be made about the meaning category\footnote{Choices can also be made for the grammar category, but we do not discuss that in this work.}. One choice is to use the category $\CPM{\FHilb}$ of Hilbert spaces and completely positive maps between them. In this category, words are represented as positive semidefinite (psd) matrices. Psd matrices have a  natural partial order called the L\"owner order, and  this order is used to model hyponymy. 
This approach was developed in \cite{sadrzadeh2018, bankova2019, lewisranlp}, and the use of psd matrices to represent words has  also been used in \cite{mots, Co20}. One of the drawbacks of this approach is that learning psd matrices from text is difficult, in particular the larger matrices that are required for functional words. Therefore, in \cite{lewisranlp, Co20},  composition rules for psd matrices have been explored. 
In \cite{Co20} these  composition rules are called \textsf{Fuzz} and \textsf{Phaser}, in \cite{lewisranlp} they are \textsf{KMult} and \textsf{BMult},  respectively. 
For this paper we stick with the guitar pedal terminology.

One of the drawbacks of these composition rules is that they do not preserve hyponymy. That is, given two pairs of words in a hyponym-hypernym relationship, the combination of the two hyponyms is not necessarily a hyponym of the combination of the two hypernyms:
\[
\lang{noun}_1 \leqslant \lang{noun}_2 \text{ and } \lang{verb}_1  \leqslant  \lang{verb}_2 \text{ does not imply } \lang{noun}_1 \ast \lang{verb}_1  \leqslant  \lang{noun}_2\ast \lang{verb}_2
\]
where the nouns and verbs are psd matrices, $\leqslant$  is the L\"owner ordering, and $\ast$ is one of \textsf{Fuzz} or \textsf{Phaser}.

The goal of this paper is to define a composition rule which is (i) positivity preserving, and (ii) hyponymy preserving. 
In addition we will require it to be  bilinear. 
If possible, it should also be non-commutative. 
Our composition rule is called Compression, \textsf{Compr}, and it is in fact an infinite set of rules; namely all completely positive maps from $\mc{M}_m$ to $\mc{M}_m\otimes \mc{M}_m$, where $\mc{M}_m$ denotes the set of real matrices of size $m\times m$.
As a special case, we recover \textsf{Mult}.

We use the following notation. 
$A^*$ denotes complex conjugate transpose, and $A_*$ means complex conjugate. 
$\PSD_m$ denotes the set of positive semidefinite (psd) matrices of size $m\times m$ over the real numbers, and a psd element is denoted by  $\geqslant 0$. 
We use the term \emph{functional words} for words such as verbs and adjectives that take arguments. Nouns are not functional words.

\section{Representing words as positive semidefinite matrices}

We assume that the reader is familiar with the categorical compositional distributional model of meaning introduced in \cite{Co10c}, Frobenius algebras as used in \cite{sadrzadehrelpron}, and the $\mathbf{CPM}$ construction \cite{selinger2007} -- we have summarised the most important ingredients in Appendix \ref{app}. We now jump right in to the representation of words as psd matrices and possible composition rules.

Positive semidefinite matrices are represented in $\CPM{\FHilb}$ as morphisms $\mathbb{R}\to M\otimes M^*$, where $M$ is some finite-dimensional Hilbert space and $M^*$ is its dual.
 The functor $\cpmsem: \mathbf{Preg} \rightarrow \CPM\FHilb$ sends nouns and sentences to psd matrices, and adjectives, verbs, and other functional words to completely positive maps, or equivalently psd matrices in a larger space. 
We represent words as psd matrices in the following way. In the vector-based model of meaning, a word $w$ is represented by a column vector, $\ket{w} \in \mathbb{R}^m$ (for some $m$). To pass to psd matrices, a subset of words $S$ will be mapped to rank 1 matrices, i.e. $\ket{w}\mapsto \ket{w}\bra{w}$. 
The words in $S$ are the hyponyms. 
The other words, which are  hypernyms of the words in  $S$, will be represented as mixtures of hyponyms:
\begin{equation}
\label{eq:psdword}
\rho =\sum_{w\in W \subset S} \ket{w}\bra{w}.
\end{equation}
Within a compositional model of meaning, we  view nouns as psd matrices in $\mc{M}_m$, and sentences as  psd matrices in $\mc{M}_s$ (for some $m$ and $s$). 
An intransitive verb has type $n^r s$ in the pregroup grammar, and is mapped by $\cpmsem$ to a psd element in $\mc{M}_m \otimes \mc{M}_s$.
Equivalently, an intransitive verb  is a completely positive (CP) map $\mc{M}_m \rightarrow \mc{M}_s$. 
The method for building psd matrices summarised in equation \eqref{eq:psdword} maps words of all grammatical types to psd matrices in $\mc{M}_m$. This is the correct type for nouns, but wrong for other grammatical types. Taking the example of intransitive verbs, we need to find a mechanism to lift an intransitive verb as a psd element in  $\mc{M}_m$ to a CP map $\mc{M}_m \rightarrow \mc{M}_s$. There have been various approaches to implemnting this type lifting, which we now summarise.


Suppose $n$ is a psd matrix for a noun, and $v$ a psd matrix for a verb. 
Proposals in \cite{lewisranlp, mots, Co20} include the following --
note in particular that \textsf{Fuzz} and \textsf{Phaser} defined in \cite{Co20} coincide with \textsf{KMult} and \textsf{BMult} defined in \cite{lewisranlp}:  
\begin{itemize}
\item 
$\textsf{Mult}(n,v) = n\odot v$ where $\odot $ is the Hadamard product, i.e.\ $(n\odot v)_{i,j} = n_{ij} v_{ij}$. 
\item 
$\textsf{Fuzz}(n,v) = \textsf{KMult} (n,v)=  \sum_ip_i P_i nP_i $ where $v=\sum_i p_i P_i$ is the spectral decomposition of $v$. That is,
$$
\textsf{Fuzz}(n,v) = \sum_{i} \sqrt{p_i} P_i n P_i \sqrt{p_i}
$$
\item 
$\textsf{Phaser}(n,v) = \textsf{BMult}(n,v)=  v^{1/2} nv^{1/2}$. 
Let $v=\sum_i p_i P_i$  be the spectral decomposition of $v$. Then
$$
\textsf{Phaser}(n,v) = \sum_{i,j} \sqrt{p_i} P_i n P_j \sqrt{p_j}
$$

\end{itemize}

Some benefits and drawbacks of these operations are as follows. \textsf{Mult} is a straightforward use of Frobenius algebra in the category $\CPM{\FHilb}$. 
It is linear, completely positive and preserves hyponymy. 
However, linguistically it is unsatisfactory because it is commutative, and so will map `Howard likes Jimmy' to the same psd matrix as `Jimmy likes Howard' -- which do not have the same meaning and so should not have the same matrix representation. On the other hand, both \textsf{Phaser} and \textsf{Fuzz} are non-commutative, however they are not linear  and do not preserve hyponymy. 
%

In the next section we outline the properties we want from a composition method, and propose a general framework that will allow us to generate a number of suggestions.

\section{In search of more guitar pedals: Compression}
\label{sec:composingpsd}

For the rest of the paper, we assume that both nouns and verbs are represented by psd matrices of the same size $m$, that is, we let  $n \in \PSD_m$ and $v\in \PSD_m$. 
We are looking for a  composition rule for these psd matrices. 
We  call the desired operation $\textsf{Compr}$ (for reasons we shall later see), and  want it to be a map
$$
\textsf{Compr}: \M_m \times  \M_m \to  \M_m. 
$$
The minimal two properties required from this map are the following: 
\begin{enumerate}
\item[(i)] {\em Positivity preserving}

If $n,v$ are psd, then $\textsf{Compr}( n,v)$ is psd: 
$$
\textsf{Compr}: \PSD_m \times \PSD_m \to \PSD_m 
$$

\item[(ii)] {\em Hyponymy preserving}

If hyponymy is represented by the L\"owner order $\leqslant$,\footnote{If $\rho$, $\sigma$ are psd, then $\rho\leqslant\sigma$ iff $\sigma -\rho\geqslant 0$, i.e.\ if $\sigma - \rho$ is itself psd.}
$$
n_1\leqslant n_2 , \quad v_1\leqslant v_2 \implies 
\textsf{Compr}(n_1,v_1)\leqslant \textsf{Compr}(n_2,v_2)
$$
\end{enumerate}

Although these two properties are the most important ones, we now  consider  another property: 

\begin{enumerate}
\item[(iii)] {\em Bilinearity}

\textsf{Compr} is linear in each of its arguments, namely for $\alpha\in \mathbb{R}$:

$\textsf{Compr}(\alpha n,v) = \alpha \textsf{Compr}(n,v)$ 

$\textsf{Compr}(n,\alpha v) = \alpha \textsf{Compr}(n,v)$ 

$\textsf{Compr}(n+n',v) =  \textsf{Compr}(n,v)+  \textsf{Compr}(n',v)$ 

$\textsf{Compr}(n,v+v') =  \textsf{Compr}(n,v)+  \textsf{Compr}(n,v')$ 
\end{enumerate}

Assumption (iii) has two advantages. 
The first one is that  if the map is positivity preserving on the Cartesian product [(i)] and bilinear [(iii)], then it is hyponymy preserving [(ii)]: 

\begin{lemma}
Assumptions (i) and (iii) imply (ii). 
\end{lemma}

\begin{proof}
Assume that $n_2\geqslant n_1$ and $v_2\geqslant 0$. Using (i) we have that $\textsf{Compr}(n_2-n_1, v_2) \geqslant 0$, and using (iii) that 
$\textsf{Compr}(n_2, v_2) \geqslant \textsf{Compr}(n_1, v_2).$ 
Now assume that $v_2\geqslant v_1$ and $n_1\geqslant 0$. Following the same argument we obtain that 
$\textsf{Compr}(n_1, v_2)\geqslant \textsf{Compr}(n_1, v_1). $
By transitivity of being psd, we obtain that
$\textsf{Compr}(n_2, v_2)\geqslant \textsf{Compr}(n_1, v_1) $, 
which is condition (ii).
\end{proof}

The second advantage of bilinearity is that it allows to reformulate \textsf{Compr} in a convenient way.
Since \textsf{Compr} is linear in both  components, we construct the following linear map:
\begin{align*}
\M_m&\to \textrm{Lin}(\M_m,\M_m)  \\ v &\mapsto  \left(\textsf{Compr}(\cdot, v)\colon n\mapsto \textsf{Compr}(n, v)\right) 
\end{align*} 
Note that linearity of  \textsf{Compr} in the noun component is necessary for the image of this map to be $\textrm{Lin}(\M_m,\M_m)$, whereas linearity  in the verb component is necessary for this map itself to be linear.
By slight abuse of notation we  denote this new map  also by \textsf{Compr}:
$$
\textsf{Compr} \colon \M_m\to  \textrm{Lin}(\M_m,\M_m).
$$ 
Now, assumption  (i) applied to this new map means that psd matrices are mapped to positivity preserving maps,  
$$
\textsf{Compr}: \PSD_m\to \textrm{PP}(\M_m,\M_m) ,
$$
where $\textrm{PP}(\M_m,\M_m)$ is the set of positivity preserving linear maps from $\M_m$ to $\M_m$, i.e.\ those that map psd matrices to psd matrices.  
To make things more tractable, one can use the isomorphism 
\begin{align*} 
\textrm{Lin}(\M_m,\M_m)&\to \M_m\otimes\M_m\\ 
\varphi &\mapsto  \sum_{i,j} \varphi(|e_i\rangle\langle e_j|)\otimes |e_i\rangle\langle e_j|,
\end{align*}
where $\{\ket{e_i}\}$ is an orthonormal basis of $\mathbb{R}^m$. 
Using this isomorphism,  $\textrm{PP}(\M_m,\M_m)$ corresponds to the set of block positive matrices  $\textrm{BP}(\M_m \otimes \M_m)$ on the tensor product space.\footnote{A matrix $\rho\in \M_m\otimes \M_m$ is block positive if $(\langle v|\otimes \langle w|) \rho (|v\rangle\otimes |w\rangle) \geq 0$ for all vectors $|v\rangle$, $|w\rangle$.}
Summarizing, we are trying to construct a linear map 
$$
\textsf{Compr} \colon \M_m\to \M_m\otimes \M_m
$$ 
that maps psd matrices to block positive matrices. So far, this is just a reformulation of conditions (i) and (iii).

 To make thinks more tractable, we now further strengthen the conditions on the map. Namely, we require  \textsf{Compr} to map psd matrices  in $\M_m$ to psd matrices in $\M_m\otimes\M_m\cong \M_{m^2}$, i.e.\ to be positivity preserving itself. 
Using the isomorphism above, this means that $\textsf{Compr}$ maps psd matrices to completely positive (CP) maps from $\M_m$ to $\M_m$  (see table \ref{tab:posmaps}): 
 $$
\textsf{Compr}: \PSD_m\to \textrm{CP}(\M_m,\M_m) . 
$$

\begin{table}[h]\centering
\begin{tabular}{rcl}
Linear map $ \M_m \to \M_m$  & $\leftrightarrow$ & Element in $\M_m\otimes \M_m$ \\
\hline 
Positivity preserving map  & $\leftrightarrow$ & Block positive matrix  \\
Completely positive map  & $\leftrightarrow$ & Positive semidefinite matrix \\
\end{tabular}
\caption{Correspondence between linear maps $ \M_m \to \M_m$  and elements in $\M_m\otimes \M_m$, known as the Choi-Jamio\l{}kowski isomorphism.}
\label{tab:posmaps}
\end{table}
 
 And since we are still not running out of steam, we  require \textsf{Compr} not only to be positivity preserving, but also to be completely positive itself. In total, we are trying to construct a {\it completely positive map} 
\begin{align}
\nonumber
\textsf{Compr}\colon \M_m &\to \M_m \otimes \M_m\cong \M_{m^2} \\ 
 v&\mapsto \sum_l K_l v K_l^*
 \label{eq:vkraus}
\end{align}
 where we have used the well known fact that every completely positive map admits a Kraus decomposition, for certain 
Kraus operators $K_l\in \mathbb R^m\otimes\M_m\cong\M_{m^2,m}$. 
Recall that ${K}^*$ denotes the complex conjugate transpose of $K$. 

In summary,  we are asking for a stronger condition than just (i) and (iii). On the other hand, 
the weaker forms of maps mentioned above do not admit a closed description, whereas completely positive maps do. 
By Stinespring's Dilation Theorem, all completely positive map can be expressed as a $*$-representation followed by a  {\it compression}, hence the name \textsf{Compr}. This is precisely what gives rise to the Kraus decomposition of the map.   This also fits well with the  electric guitar pedal notation from \cite{Co20}, as can be seen in figure \ref{Fig:Comp}.

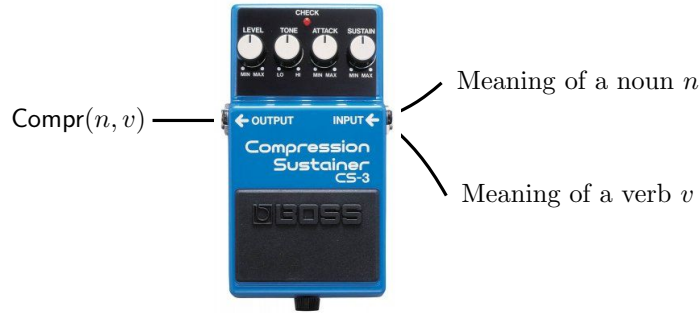
\begin{figure}[h]
\centering
\begin{tikzpicture}
	\begin{pgfonlayer}{nodelayer}
		\node [style=none, align=right] (1) at (3.5, 1) {$\:$ Meaning of a noun $n$};
		\node [style=none] (2) at (3.5, -.5) {$\:$ Meaning of a verb $v$};
		\node [style=none] (3) at (1, 0.5) {};
		\node [style=none] (5) at (-1, .5) {};
		\node [style=none] (6) at (-3, .5) { $\textsf{Compr}(n,v)\:$};
		\node [style=none] (7) at (0,0)
{\includegraphics[width=.15\textwidth]{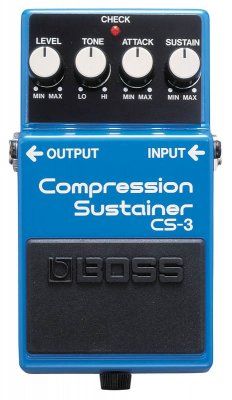}};
	\end{pgfonlayer}
	\begin{pgfonlayer}{edgelayer}
		\draw [very thick, bend right=-10] (1.west) to (3.north);
		\draw [very thick, bend right=10] (2.west) to (3.north);
		\draw [very thick] (5.center) to (6.east);
	\end{pgfonlayer}
\end{tikzpicture}
\caption{Not only Fuzz and Phaser, but also Compression is an important guitar pedal. The operation \textsf{Compr} takes as input an element of $\M_m \times \M_m$ (denoted $n, v$) representing the meaning of a noun and the meaning of a verb, and it outputs $\textsf{Compr}(n,v)$, representing the meaning of the sentence $n \: v$. }
\label{Fig:Comp}
\end{figure}

Note that in general such \textsf{Compr} is non-commutative, i.e.\ when translated back to the initial setup of  $$\textsf{Compr}: \M_m \times  \M_m \to  \M_m $$ we will generally have $\textsf{Compr}( n,v) \neq \textsf{Compr}( v ,n) $. This is a good property, as it reflects the position of the words in a sentence has both syntactic and semantic roles,  e.g.\ `woman bites dog' versus `dog bites woman'. 
Note also that  \textsf{Fuzz} and \textsf{Phaser} (or \textsf{KMult} and \textsf{BMult}) are not linear in the verb component, i.e.\ they do not fulfill (iii), and are thus  not special cases of \textsf{Compr}.  However, \textsf{Mult} is a special case of \textsf{Compr}, as we shall see. 

\section{Building nouns and verbs in $\CPM{\FHilb}$} 
In order to build psd matrices to represent words, we can use pretrained word embeddings such as word2vec \cite{word2vec} or GloVe \cite{glove}, together with information about hyponymy relations. The word embedding methods produce vectors for each word, all represented in one vector space $W$. Information about hyponymy relations can be found in WordNet \cite{wordnet} or in a less supervised manner by extracting hyponym-hypernym pairs using Hearst patterns \cite{hearst1992, hearst}. 

Given a word $w$, we can gather a set of hyponyms  $\{h_i\}_i$ from WordNet, Hearst patterns, or some other source. We then take the vectors for the $h_i$ from pretrained word embeddings, and form the matrix
\[
\sem{w} =  \sum_i \ket{h_i}\bra{h_i} \in W \otimes W^*.
\]
where $W^*$ is the dual of the column vector space $W$. $\sem{w}$ is then normalised. In this work, we normalise using the infinity norm,  that is, we divide by the maximum eigenvalue. This has been shown to have  nice properties \cite{vdwetering}. 
%
%

This approach to building word representations puts all word representations in the shared space $W \otimes W^*$. If we are working in the category $\CPM{\FHilb}$, this is the right kind of representation for nouns,  but not for functional words. 
%
To transform a psd matrix $\sem{verb} \in \PSD_m$ to a psd matrix in $\mc{M}_{m^2}$, we use the composition rule  \textsf{Compr} proposed in section \ref{sec:composingpsd}:
\begin{align*} 
\textsf{Compr}\colon \M_m&\to \M_m\otimes\M_m\cong\M_{m^2} \\
v&\mapsto \textsf{Compr}(\cdot,v) 
\end{align*}

\subsection{Characterising $\mathsf{Compr}$ diagrammatically}

We now characterise $\mathsf{Compr}$, $\mathsf{Compr}(\cdot,v)$, and $\mathsf{Compr}(n,v)$ in the diagrammatic calculus for $\FHilb$. This will allow us to generate simple examples of $\mathsf{Compr}$ in a systematic manner. 
%
%
Equation \eqref{eq:vkraus} states:
\[
\mathsf{Compr}(\cdot,v) = \sum_l K_l v K_l^*
\] 
Diagrammatically, this gives us:

\[
\mathsf{Compr} = \begin{gathered} \input{tikz/T_mixed.tikz}\end{gathered}, \qquad \mathsf{Compr}(\cdot,v) =  \begin{gathered} \input{tikz/Tv_mixed.tikz}\end{gathered}
\]
The application of $\mathsf{Compr}(\cdot,v)$ to $n$ is then:
\begin{equation}
\label{eq:tvn}
\mathsf{Compr}(n,v) = \begin{gathered} \input{tikz/Tvn.tikz}\end{gathered}
\end{equation}
Note that this style corresponds to the usual DisCoCat diagram style via some reshaping, explained in  equation \eqref{eq:tscpmc} of the appendix.
%
%
%
%
Given that we have representations of $v$ and of $n$, what should the $K_l$ look like? In full generality, parameters of the $K_l$ could be perhaps inferred using regression techniques,
 in a similar approach to that suggested in \cite{lewisrnns}, inspired by \cite{socher}, or via methods like those in \cite{baroni2010, grefenstette2013}. However, we can also consider purely ``structural'' morphisms, generated from cups, caps, swaps, and spiders (for details, see table \ref{tab:frob} in the appendix). In the following, we give a number of options  to specify $K$. We divide up the internal structure of $K$ by specifying the number of spiders inside $K$.

  \paragraph{0 spiders}
  \begin{align}
     \label{eq:trnv}
   K &= \begin{gathered}\input{tikz/Kid.tikz}\end{gathered}, \quad \mathsf{Compr}(v)(n) \:=\:  \begin{gathered}\input{tikz/Tvnidfull.tikz}\end{gathered} \:=\:  \begin{gathered}\input{tikz/Tvnid.tikz}\end{gathered} = \tr(n) v\\
   K &= \begin{gathered}\input{tikz/Kswap.tikz}\end{gathered}, \quad \mathsf{Compr}(v)(n) \:=\:  \begin{gathered}\input{tikz/Tvnswapfull.tikz}\end{gathered} \:=\:  \begin{gathered}\input{tikz/Tvnswap.tikz}\end{gathered} = \tr(nv) \mathbb{I}\\
      \label{eq:trvn}
   K &= \begin{gathered}\input{tikz/Kcupcap.tikz}\end{gathered}, \quad \mathsf{Compr}(v)(n) \:=\:  \begin{gathered}\input{tikz/Tvncupcapfull.tikz}\end{gathered} \:=\:  \begin{gathered}\input{tikz/Tvncupcap.tikz}\end{gathered} = \tr(v) n
 \end{align}
 
  \paragraph{1 spider}
  \begin{align}
         \label{eq:diagndiagv}
    K &= \begin{gathered}\input{tikz/Kall.tikz}\end{gathered}, \quad \mathsf{Compr}(v)(n) \:=\:  \begin{gathered}\input{tikz/Tvnallfull.tikz}\end{gathered} \:=\:  \begin{gathered}\input{tikz/Tvnall.tikz}\end{gathered} = \mathrm{diag}(n) \mathrm{diag}(v)
 \end{align}
 
 \paragraph{2 spiders}
  \begin{align}
    \label{eq:nsumv}
 K &= \begin{gathered}\input{tikz/KLupRup.tikz}\end{gathered}, \quad \mathsf{Compr}(v)(n) \:=\:  \begin{gathered}\input{tikz/TvLupRupfull.tikz}\end{gathered} \:=\:  \begin{gathered}\input{tikz/TvLupRup.tikz}\end{gathered} = mn\sum_{ij} v_{ij}\\
  K &= \begin{gathered}\input{tikz/KLupRdown.tikz}\end{gathered}, \quad \mathsf{Compr}(v)(n) \:=\:  \begin{gathered}\input{tikz/TvnLupRdownfull.tikz}\end{gathered} \:=\:  \begin{gathered}\input{tikz/TvnLupRdown.tikz}\end{gathered} = m\tr(nv)\sum_{ij} \ket{e_i}\bra{e_j}\\
   \label{eq:vsumn}
  K &= \begin{gathered}\input{tikz/KLupLdown.tikz}\end{gathered}, \quad \mathsf{Compr}(v)(n) \:=\:  \begin{gathered}\input{tikz/TvnLupLdownfull.tikz}\end{gathered} \:=\:  \begin{gathered}\input{tikz/TvnLupLdown.tikz}\end{gathered} = mv\sum_{ij} n_{ij}\\
  K &= \begin{gathered}\input{tikz/KRupRdown.tikz}\end{gathered}, \quad \mathsf{Compr}(v)(n) \:=\:  \begin{gathered}\input{tikz/TvnRupRdownfull.tikz}\end{gathered} \:=\:  \begin{gathered}\input{tikz/TvnRupRdown.tikz}\end{gathered} = \tr(n)\sum_{ij} v_{ij}\sum_{kl}\ket{e_k}\bra{e_l}
   \end{align}
 \begin{align}
 K &= \begin{gathered}\input{tikz/KRupLdown.tikz}\end{gathered}, \quad \mathsf{Compr}(v)(n) \:=\:  \begin{gathered}\input{tikz/TvnRupLdownfull.tikz}\end{gathered} \:=\:  \begin{gathered}\input{tikz/TvnRupLdown.tikz}\end{gathered} = \mathbb{I} \sum_{ij} n_{ij}\sum_{kl}v_{kl}\\
  K &= \begin{gathered}\input{tikz/KRdownLdown.tikz}\end{gathered}, \quad \mathsf{Compr}(v)(n) \:=\:  \begin{gathered}\input{tikz/TvnRdownLdownfull.tikz}\end{gathered} \:=\:  \begin{gathered}\input{tikz/TvnRdownLdown.tikz}\end{gathered} = \tr(v) \sum_{ij} n_{ij}\sum_{kl}\ket{e_k}\bra{e_l}\\
  K &= \begin{gathered}\input{tikz/Kmuiota.tikz}\end{gathered}, \quad \mathsf{Compr}(v)(n) \:=\:  \begin{gathered}\input{tikz/Tvnmuiotafull.tikz}\end{gathered} \:=\:  \begin{gathered}\input{tikz/Tvnmuiota.tikz}\end{gathered} = \tr\left( \sum_{ij} n_{ij}v_{ij} \ket{e_i}\bra{e_j}\right)\sum_{kl}\ket{e_k}\bra{e_l}\\
       \label{eq:sumndiagv}
  K &= \begin{gathered}\input{tikz/Kiotamu.tikz}\end{gathered}, \quad \mathsf{Compr}(v)(n) \:=\:  \begin{gathered}\input{tikz/Tvniotamufull.tikz}\end{gathered} \:=\:  \begin{gathered}\input{tikz/Tvniotamu.tikz}\end{gathered} = \sum_{ij} n_{ij}\mathrm{diag}(v)
  \end{align}
  \begin{align}
       \label{eq:sumvdiagn}
   K &= \begin{gathered}\input{tikz/KDeltazeta.tikz}\end{gathered}, \quad \mathsf{Compr}(v)(n) \:=\:  \begin{gathered}\input{tikz/TvnDeltazetafull.tikz}\end{gathered} \:=\:  \begin{gathered}\input{tikz/TvnDeltazeta.tikz}\end{gathered} = \sum_{ij} v_{ij}\mathrm{diag}(n)
 \end{align}
 \begin{align}
        \label{eq:mult}
    K &= \begin{gathered}\input{tikz/KzetaDelta.tikz}\end{gathered}, \quad \mathsf{Compr}(v)(n) \:=\:  \begin{gathered}\input{tikz/TvnzetaDeltafull.tikz}\end{gathered} \:=\:  \begin{gathered}\input{tikz/TvnzetaDelta.tikz}\end{gathered} = \sum_{ij} v_{ij}n_{ij} \ket{e_i}\bra{e_j} = \mathsf{Mult}(n, v)
 \end{align}
 
  \paragraph{3 spiders} The instances with 3 spiders are subsumed by the instances with 2 spiders, since to have 3 spiders we would need two spiders with one leg and one spider with two legs. A spider with two legs is either a cup, cap, or the identity morphism, hence these have been included in the 2 spider instances.
 
  \paragraph{4 spiders}
 
 \begin{align}
 K = \begin{gathered}\input{tikz/K4.tikz}\end{gathered}, \quad \mathsf{Compr}(v)(n) \:=\:  \begin{gathered}\input{tikz/Tvn4full.tikz}\end{gathered} \:=\:  \begin{gathered}\input{tikz/Tvn4.tikz}\end{gathered} = m\sum_{ij} n_{ij}\sum_{kl} v_{kl}\sum_{rs}\ket{e_r} \bra{e_s}
 \end{align}
 
 This gives us a whole range of possible instantiations of $\mathsf{Compr}$. Some of these options are more interesting than others. Options that give us a multiple of the identity matrix or a multiple of $\sum_{ij} \ket{e_i}\bra{e_j} $ for orthonormal basis $\{\ket{e_i}\}_i$ are less interesting since this means that all phrase representations will be mapped to the same psd matrix, differing only by a scalar. This means that although hyponymy information may be preserved, information about similarity will be lost.
 
 In the following section we test out a number of options on some phrase entailment datasets.  We test equations: \eqref{eq:trnv}, \eqref{eq:trvn}, \eqref{eq:diagndiagv}, \eqref{eq:nsumv} \eqref{eq:vsumn}, \eqref{eq:sumndiagv}, \eqref{eq:sumvdiagn}, and \eqref{eq:mult}, the last of which was already shown to work well in \cite{lewisranlp}.

\section{Demonstration}
To test these composition methods, we follow the setup in \cite{lewisranlp, lewisnegation}. We firstly build psd matrices using GloVe vectors. For this small scale demonstration we use GloVe vectors of dimension 50. We use a set of datasets that contain pairs of short phrases, for which the first either does or does not entail the second. 

In addition, we use a graded form of the L\"owner ordering to measure hyponymy, since in general the crisp L\"owner ordering will not be obtained between two psd matrices $A$ and $B$. This graded form is measured as follows. 
Given two psd matrices $A$ and $B$, if $A \leqslant B$ then $ A + D = B $ where $D$ is itself a psd matrix. If this does not hold, we can add an error term $E$ so that 
\[ 
A + D = B + E.
\]  
In the worst case, we can set $E = A$, so that $D = B$, and in fact we will always have that $E \leqslant A$. A graded measure of hyponymy is obtained by comparing the size of $E$ and $A$. We set \[k_E = 1 - \frac{||E||}{||A||},\] 
where $||\cdot||$ denotes the Euclidean norm, $||A|| =\sqrt{\tr(A^*A)}$.  
{The crisp L\"owner order is recovered in the case} that $E = 0$, so that $k_E = 1$. 
A second measure of graded hyponymy is obtained as follows: 
\begin{equation}
k_{BA} = \frac{\sum_i \lambda_i}{\sum_i |\lambda_i|}
\end{equation}
where $\lambda_i$ is the $i$th eigenvalue of $B - A$ and $| \cdot |$ indicates absolute value. This measures the proportions of positive and negative eigenvalues in the expression $B-A$. If all eigenvalues are negative, $k_{BA} = -1$, and if all are positive, $k_{BA} = 1$. This measure is balanced in the sense that $k_{BA} = - k_{AB}$.

\paragraph{Datasets}
The datasets were originally collected for \cite{kartsaklis2016}. They consist of ordered pairs of short phrases in which the first entails the second, and also the same pair in the opposite order, so that the first phrase does not entail the second. The datasets were gathered using WordNet as source. The datasets contain intransitive sentences, of the form \lang{subject verb}, verb phrases, of the form \lang{verb object} and transitive sentences, of the form \lang{subject verb object}.  For example:
\begin{quote}
summer finish, season end, \texttt{true}

season end, summer finish, \texttt{false}
\end{quote}

The datasets have a binary classification, so we measure performance using area under receiver operating characteristic (ROC) curve. If we imagine that our graded measure is converted to a binary measure by giving a threshold, area under ROC curve measures performance at all cutoff thresholds. A value of 1 means that the graded values are in fact a completely correct binary classification, a value of 0.5 means that the graded values are randomly correlated with the correct classification, and a value of 0 means that the graded values are binary values that are classified in exactly the wrong way (a value of 1 is mapped to 0 and 0 to 1).


\paragraph{Models}

We test the following models, for $n, v \in \M_m$. 
We denote by $\diag(A)$ the matrix obtained by setting all off-diagonal elements of $A$ to 0.  
In order to retain the property that the maximum eigenvalue is less than or equal to 1, we divide by the dimension $m$ or by $m^2$ where necessary.
\begin{enumerate}
\item \textbf{Traced noun}: $\textsf{Compr}(n, v) = \frac{\tr(n)}{m} v$
\item \textbf{Traced verb}: $\textsf{Compr}(n, v) = \frac{\tr(v)}{m} n$
\item \textbf{Diag}: $\textsf{Compr}(n, v) = \diag(n)\diag(v)$ 
\item \textbf{Summed noun}: $\textsf{Compr}(n, v) = \frac{v}{m^2} \sum_{ij} n_{ij}$
\item \textbf{Summed verb}: $\textsf{Compr}(n, v) = \frac{n}{m^2} \sum_{ij} v_{ij}$
\item \textbf{Diag verb}: $\textsf{Compr}(n, v) = \frac{\diag(v)}{m^2} \sum_{ij} n_{ij}$
\item \textbf{Diag noun}: $\textsf{Compr}(n, v) = \frac{\diag(n)}{m^2} \sum_{ij} v_{ij}$
\item \textbf{Mult}: $\textsf{Compr}(n, v) = \sum_{ij} v_{ij}n_{ij} \ket{e_i}\bra{e_j}$
\end{enumerate}
Above, we have specified models for sentences of the form \lang{subj verb}. For verb phrases, we treat the verb as $v$ and the object as $n$, so the models differ based on the grammatical type of the word, rather than its position in the argument list. For sentence of the form \lang{subject verb object}, we first combine the verb and the object, according to their grammatical type, and then treating this verb phrase as an intransitive verb, combine the subject and verb phrase, again according to grammatical type. So, for example, iterating the composition \textbf{Traced Verb} on psd matrices $s$, $v$, $o$ for subject, verb, and object, we obtain:
\[
\textsf{Compr}(s, \textsf{Compr}(o, v)) = \textsf{Compr}(s, \frac{\tr(v)}{m} o) = \frac{\tr(v)\tr(o)}{m^2} s
\]
We also test two combined models:
\begin{enumerate}
\item \textbf{Traced addition}: $\textsf{Compr}(n, v) = \frac{\tr(n)}{2m} v + \frac{\tr(v)}{2m} n$
\item \textbf{Summed addition}: $\textsf{Compr}(n, v) = \frac{v}{m^2} \sum_{ij} n_{ij} +  \frac{n}{m^2} \sum_{ij} v_{ij}$
\end{enumerate}
We compare with a verb-only baseline and with \textsf{Fuzz} and \textsf{Phaser}. These last two are tested in two directions: 
\begin{enumerate}
\item \textbf{Verb only}: $\textsf{Verb only}(n, v) = v$ 
\item \textbf{Fuzz}: $\textsf{Fuzz}(n, v) = \sum_{i} \sqrt{p_i} P_i n P_i \sqrt{p_i}$ where $\sum_i p_i P_i$ is the spectral decomposition of $v$
\item \textbf{Fuzz switched}: $\textsf{Fuzz-s}(n, v) = \sum_{i} \sqrt{q_i} Q_i v Q_i \sqrt{q_i}$ where $\sum_i q_i Q_i$ is the spectral decomposition of $n$ 
\item \textbf{Phaser}: $\textsf{Phaser}(n, v) = \sqrt{v} n \sqrt{v}$ 
\item \textbf{Phaser switched}: $\textsf{Phaser-s}(n, v) = \sqrt{n} v \sqrt{n}$
\end{enumerate}

To test for significance of our results, we bootstrap the data with 100 repetitions \cite{efron1992} and compare between models using a two sample t-test. We apply the Bonferroni correction to compensate for multiple comparisons.

\subsection{Results}
Results are presented in table \ref{tab:results}. A key point is that as in previous work, the $k_{BA}$ measure performs better than the $k_E$ measure. Furthermore, across all datasets, the models \textbf{Traced verb} and \textbf{Summed verb} perform much more highly than simply taking the verb on its own, indicating that information about at least the size of the noun is crucial. Across both measures, performance is highest on the SVO dataset and lowest on the VO dataset. This may be due to the construction of the datasets, or it may be due to these composition methods working well on longer phrases.

 Within the results using the $k_E$ measure, the model \textbf{Diag verb} is strong across all datasets. This is surprising, as taking the diagonal of the verb would seem to result in information loss. Within the results using the $k_E$ measure, the new models largely outperform \textbf{Phaser}, and on the VO dataset, outperform \textbf{Fuzz} too. 
 
Within the results using the $k_{BA}$ measure, the picture is less clear. \textbf{Diag}, \textbf{Mult}, and \textbf{Traced addition} are all fairly strong, but there is no outright best model. Perhaps looking at some other combination possibilities would be useful.

\begin{table}[htbp]
\begin{center}
\begin{tabular}{l | l l l | l l l }
& \multicolumn{3}{|c|}{$k_E$ measure} & \multicolumn{3}{|c}{$k_{BA}$ measure}\\
& SV & VO & SVO& SV & VO & SVO\\
\textbf{Verb only} 		& 0.599 		& 0.586 		& 0.652 		& 0.787 		& 0.744 		& 0.834 \\\hline
\textbf{Fuzz}			& 0.867 		& 0.803 		& 0.917 		& 0.927 		& 0.896 		& 0.968\\
\textbf{Fuzz switched}	& 0.809 		& 0.743 		& 0.940 		& 0.934 		& 0.891 		& 0.953\\
\textbf{Phaser} 			& 0.833 		& 0.761 		& 0.925 		& 0.924 		& 0.896 		& 0.970 \\
\textbf{Phaser switched} 	& 0.765 		& 0.717 		& 0.932 		& 0.930 		& 0.891 		& 0.977\\\hline
\textbf{Traced noun}		& 0.813 		& 0.769 		& 0.933 		& 0.936		& $0.912^{*+}$ 	& 0.960\\
\textbf{Traced verb}		& 0.842 		& $0.803^+$ 		& $0.949^{*+}$ 		& 0.930 		& $0.909^*$ 		& $0.974^*$\\
\textbf{Diag}			& $0.898^{*+}$ 		& $0.860^{*+}$ 		& $0.943^+$ 		& $0.937^+$ 	&$\textbf{0.916}^{*+}$ 	& 0.967\\
\textbf{Summed noun}	& 0.794 		& $0.779^+$ 		& 0.898 		& 0.890 		& 0.886 		& 0.933\\
\textbf{Summed verb}	& $0.865^+$ 		& $0.810^+$ 		& 0.936 		& 0.926 		& 0.884 		& 0.970 \\
\textbf{Diag verb}		&$\textbf{0.916}^{*+}$ 	&$\textbf{0.876}^{*+}$	&$\textbf{0.977}^{*+}$	& 0.917 		& 0.875 		& 0.971\\
\textbf{Diag noun}		& $0.872^+$ 		& $0.852^{*+}$ 		& 0.432 		& 0.881 		& 0.875 		& 0.870 \\
\textbf{Mult}			& $0.850^+$ 		& $0.813^{*+}$ 		& $0.941^+$ 		&$\textbf{0.943}^{*+}$ 	& $0.915^{*+}$ 		& 0.969 \\
\textbf{Traced addition} 	& $0.868^+$ 		& $0.830^{*+}$ 		& $0.964^{*+}$ 		& 0.934 		& $0.909^{*+}$ 		&$ \textbf{0.985}^{*+}$\\
\textbf{Summed addition} 	& $0.854^+$ 		& $0.821^{*+}$ 		& 0.937 		& 0.917 		& 0.896 		& 0.966\\
\end{tabular}
\end{center}
\caption{Area under ROC curve for $k_E$ and $k_{BA}$ graded hyponymy measures. Figures are mean values of 100 samples taken from each dataset with replacement. $-^*$ indicates significantly better than both variants of \textbf{Fuzz}, $p < 0.01$, $-^+$ indicates significantly better than both variants of \textbf{Phaser}, $p < 0.01$.}
\label{tab:results}
\end{table}%

\section{Discussion}

We have presented a general composition rule called $\textsf{Compr}$ for converting a psd matrix for a functional word such as a verb or an adjective into a CP map that matches the grammatical type of the word. 
$\textsf{Compr}$ preserves hyponymy, in contrast to previous approaches like  $\textsf{Fuzz}$ and $\textsf{Phaser}$. While in full generality  we would want to learn the parameters of $\textsf{Compr}$ from a text corpus, as a first step we have defined the structure of $\textsf{Compr}$ using just cups, caps, and spiders. Results on the text datasets are promising, although there is no completely clear advantage over  $\textsf{Fuzz}$ or $\textsf{Phaser}$.

The approach we have taken, namely that of defining a map that converts representations of functional words to a higher-order type, has also been seen in vector-based models of meaning. 
In \cite{kartsaklis2012}, \cite{grefenstette2011}, word vectors and also matrices are converted using Frobenius algebras, of which the composition \textsf{Mult} is a direct analogue. 
Furthermore, in \cite{mitchell2010}, and recapitulated in \cite{lewisrnns}, a bilinear map $C: N \otimes N \rightarrow N$  that gives the composition of two vectors is proposed. Under this approach, we would have
\[
\sem{\lang{subj verb}} = C(\sem{\lang{subj}} \otimes \sem{\lang{verb}})
\text{ and }
\sem{\lang{subj verb obj}} = C(\sem{\lang{subj}} \otimes C(\sem{\lang{verb}} \otimes \sem{\lang{obj}}))
\]
Our approach is an analogue to this one within the realm of psd matrices and CP maps.

There are a number of strands to this work to be continued. We would like to learn $\textsf{Compr}$ directly from text, rather than specifying the structure by hand. More freedom in the parameters of  $\textsf{Compr}$ means that we could define it as a CP map 
\[
\textsf{Compr}\colon \M_m\to \M_m\otimes\M_s.
\]
where we are then matching the grammatical types more exactly.
In addition, the psd matrices we use are built using human curated resources -- ideally these would be learnt in a less supervised manner directly from text corpora.
At present, we have given two possible graded measures of hyponymy -- more research into these measures is needed, including how they interact with the composition methods we have specified.
 Work is currently ongoing to develop a model of negation within this framework \cite{lewisnegation}.

\bibliographystyle{plain}
\bibliography{biblio/all-my-bibliography.bib}

\appendix{
\section{Categorical compositional distributional semantics}
\label{app}

In this appendix we give a brief introduction to the categorical compositional approach to distributional semantics -- for details see \cite{Co10c, kartsaklis2012}. 

\subsection{Compositional distributional semantics}
 
 We start by reviewing some of the category theory used in categorical compositional models of meaning.


\begin{dfn}
A \define{monoidal category} is a tuple $(\mathbf{C}, \otimes, I)$ where
\begin{itemize}
\item $\mathbf{C}$ is a category, meaning that:
\begin{itemize} 
\item $\mathbf{C}$ has a collection of \define{objects} $A, B, ...$ and each ordered pair of objects $(A, B)$ has a collection of \define{morphisms} $f: A \rightarrow B$
\item For each triple of objects $(A,B, C)$ and morphisms $f: A \rightarrow B$, $g: B \rightarrow C$ there is a \define{sequential composite} $g \circ f: A \rightarrow C$ that is associative, i.e. 
\[
h \circ(g\circ f) = (h \circ g) \circ f
\]
\item for each object $A$ there is an \define{identity morphism} $1_A: A \rightarrow A$ such that  for $f: A \rightarrow B$

\[
f \circ 1_A = f \quad and \quad 1_B \circ f = f 
\]
\end{itemize}
\item for each ordered pair of objects $(A, B)$, there is a \define{composite object} $A \otimes B$, and we moreover require that:
\[
A \otimes(B \otimes C) \cong (A\otimes B) \otimes C
\]
where $\cong$ means `is isomorphic to'.
\item there is a \define{unit object} $I$, which satisfies
\[
I \otimes A \cong A \cong A \otimes I
\]
\item for each ordered pair of morphisms $f: A \rightarrow B$, $g: B \rightarrow C$ there is a \define{parallel composite} $f \otimes g : A \otimes B \rightarrow C \otimes D$ which satisfies:
$$
(g_1 \otimes g_2) \circ (f_1 \otimes f_2) = (g_1 \circ f_1) \otimes (g_2 \circ f_2) 
$$
\end{itemize}
\end{dfn}

For a precise statement and discussion of the above definition, we direct the reader to~\cite{MacLane1971}, and \cite{CoeckePaquette2011} for a more gentle introduction.  Monoidal categories can be given a graphical calculus as in figure \ref{fig:mongraph}.

\begin{figure}[h!]
\centering
\input{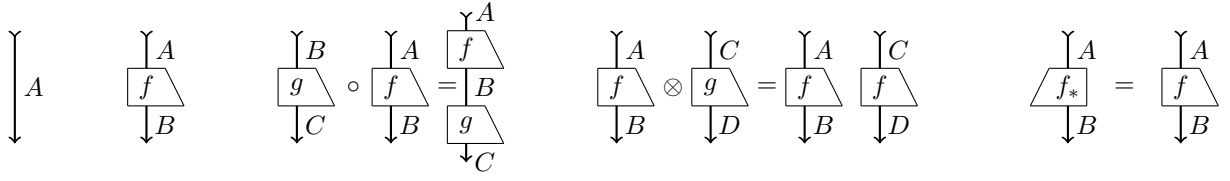}
\caption{Monoidal graphical calculus.}
\label{fig:mongraph}
\end{figure}
By convention the wire for the monoidal unit is omitted.
As will be seen in section \ref{sec:preg}, we will require that the grammar category and the meaning category are both a particular kind of monoidal category, called \emph{compact closed}. 

\begin{dfn}
A monoidal category $(C,\otimes, I)$ is \define{compact closed} if for each object $A\in C$ there are objects $A^l,A^r\in C$ (the \define{left} and \define{right duals} of $A$) and morphisms
\begin{align*}
\eta_A^l : I\rightarrow A\otimes A^l, \quad \eta_A^r : I\rightarrow A^r\otimes A, \quad \epsilon_A^l : A^l\otimes A \rightarrow I, \quad \epsilon_A^r : A\otimes A^r\rightarrow I
\end{align*}
satisfying the snake equations
\begin{align*}
&(1_A \otimes \epsilon^l_A )\circ (\eta^l_A \otimes 1_A) = 1_A &\:& (\epsilon^r_A \otimes 1_A) \circ (1_A \otimes \eta^r_A ) = 1_A
\\
&(\epsilon^l_A \otimes 1_{A^l})\circ (1_{A^l} \otimes \eta^l_A)= 1_{A^l} &\:& (1_{A^r} \otimes \epsilon^r_A ) \circ (\eta^r_A \otimes 1_{A^r} ) = 1_{A^r}
\end{align*}
\end{dfn}

The $\epsilon$ and $\eta$ maps are called \define{cups} and \define{caps} respectively, and can also be depicted in the graphical calculus as in figure \ref{fig:comgraph}. The snake equations are depicted graphically as in figure \ref{fig:snake}.

\begin{figure}[h!]
\centering
\input{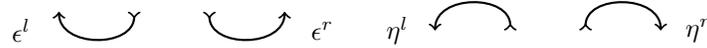}
\caption{Compact structure graphically.}
\label{fig:comgraph}
\end{figure}

\begin{figure}[h!]
\centering
\input{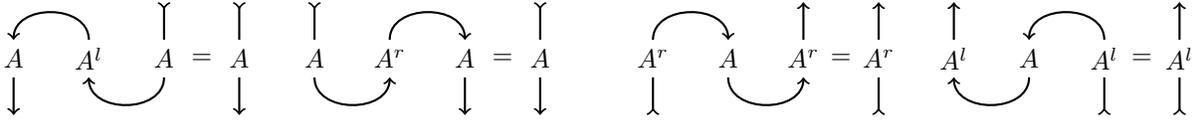}
\caption{The snake equations.}
\label{fig:snake}
\end{figure}

%

Finally, we introduce the notion of a Frobenius algebra over a real finite-dimensional Hilbert space. For a mathematically rigorous presentation see~\cite{sadrzadeh2013}.
A real Hilbert space with a fixed orthonormal basis~$\{\ket{v_i}\}_i$ has a Frobenius algebra given by:
\[
\Delta: V \rightarrow V \otimes V ::\ket{v_i} \mapsto \ket{v_i} \otimes \ket{v_i} \quad \iota : V \rightarrow \mathbb{R} :: \ket{v_i} \mapsto 1 
\]
\[
\mu : V \otimes V \rightarrow V :: \ket{v_i} \otimes \ket{v_j}
\mapsto \delta_{ij} \ket{v_i} \quad \zeta : \mathbb{R} \rightarrow V :: 1 \mapsto \sum_i \ket{v_i}
\]
This algebra is commutative, so for the swap map~$\sigma: X \otimes Y \rightarrow Y\otimes X$, we have~$\sigma \circ \Delta = \Delta$ and~$\mu \circ \sigma = \mu$.
It is also special so that~$\mu \circ \Delta = 1$. Essentially, the~$\mu$ morphism amounts to taking the diagonal of a matrix,
and~$\Delta$ to embedding a vector within a diagonal matrix.

Diagrammatically, these are represented as follows:

\begin{equation*}
\Delta = 
\begin{gathered}
\input{tikz/Delta.tikz}
\end{gathered},\qquad
\iota = 
\begin{gathered}
\begin{tikzpicture}[scale=0.5]
	\begin{pgfonlayer}{nodelayer}
		\node [style=small circ] (0) at (0, -0.5) {};
		\node [style=none] (1) at (0, 1) {$V$};
		\node [style=none] (2) at (0, 0.5) {};
	\end{pgfonlayer}
	\begin{pgfonlayer}{edgelayer}
		\draw (2.center) to (0);
	\end{pgfonlayer}
\end{tikzpicture}
\end{gathered}, \qquad
\mu = 
\begin{gathered}
\input{tikz/mu.tikz}
\end{gathered},\qquad
\zeta = 
\begin{gathered}
\begin{tikzpicture}[scale=0.5]
	\begin{pgfonlayer}{nodelayer}
		\node [style=small circ] (0) at (0, 0.5) {};
		\node [style=none] (1) at (0, -1) {$V$};
		\node [style=none] (2) at (0, -0.5) {};
	\end{pgfonlayer}
	\begin{pgfonlayer}{edgelayer}
		\draw (2.center) to (0);
	\end{pgfonlayer}
\end{tikzpicture}
\end{gathered}
\end{equation*}

Frobenius structures can have any number of wires in and out, and moreover, fuse together, so that the only thing that matters about these morphisms are how many wires in and out they have.
%

\subsection{Meaning: from vectors to positive semidefinite matrices}
\label{sec:cpm}

We wish to model words as positive semidefinite matrices. For this to work within the categorical compositional approach, we require that psd matrices have a home within a compact closed category.
We can provide this home by using the $\mathbf{CPM}$ construction \cite{selinger2007}, applied to $\FHilb$. Throughout this section~$\mathcal{C}$ denotes an arbitrary $\dag$-compact closed category. 
\begin{dfn}[Completely positive morphism \cite{selinger2007}]
  A $\mathcal{C}$-morphism~$\varphi: A^* \otimes A \rightarrow B^* \otimes B$ is said to be \emph{completely positive}~if there exists~$C \in \mathsf{Ob}(\mathcal{C})$
  and~$k \in \mathcal{C}(C\otimes A, B)$, such that~$\varphi$ can be written in the form:
\[
 (k_* \otimes k) \circ (1_{A^*} \otimes \eta_C \otimes 1_A)
\] 
where in $\FHilb$, $k_*$ is the complex conjugate of $k$.
\end{dfn}
 Identity morphisms are completely positive, and completely positive morphisms are closed under composition in~$\mathcal{C}$, leading to the following:
\begin{dfn}
  If~$\mathcal{C}$ is a $\dag$-compact closed category then~$\CPMC$ is a category with the same objects as~$\mathcal{C}$ and its morphisms are the completely positive morphisms.
\end{dfn}
Note that morphisms are defined for objects of the form $A^*\otimes A$, where $A$ is an object in $\mathcal{C}$.
The $\dagger$-compact structure required for interpreting language in our setting lifts to~$\CPM{\mathcal{C}}$:
\begin{theorem}(\cite{selinger2007})
  $\CPM{\mathcal{C}}$ is also a $\dagger$-compact closed category.
  There is a functor:
  \begin{align*}
    \cpmpure : \mathcal{C} &\rightarrow \CPM{\mathcal{C}}\\
    k &\mapsto k_* \otimes  k
  \end{align*}
  where $k_*$ denotes the complex conjugate of $k$. This functor preserves the $\dagger$-compact closed structure, and is faithful ``up to a global phase''.   
\end{theorem}
It follows that applying the $\mathbf{CPM}$ construction to $\FHilb$, we obtain a $\dagger$-compact closed category. The objects of the category are finite dimensional Hilbert spaces, and the morphisms are completely positive maps. The compact closed structure is summarised in table \ref{tab:dc}. Objects of $\CPM{\FHilb}$ can be given a Frobenius algebra, summarised in table \ref{tab:frob}.
\begin{table}[h!]
\centering
\caption{Table of diagrams in $\CPMC$ and $\mathcal{C}$}
\label{tab:dc} 
\begin{tabular}{ c c}
  $\CPMC$ & $\mathcal{C}$\\
  \hline
  $E(\epsilon) = \epsilon_* \otimes \epsilon$ & $\epsilon : A^* \otimes A^* \otimes A \otimes A \rightarrow I$ \\
  \input{tikz/cup_cpmc.tikz}&\input{tikz/cup_c.tikz}\\
  \multicolumn{2}{c}{$\epsilon: \ket{e_i} \otimes \ket{e_j} \otimes \ket{e_k} \otimes \ket{e_l} \mapsto \braket{e_i | e_k}\braket{e_j| e_l}$} \\
&\\
 $E(\eta) = \eta_* \otimes \eta$ &  $\eta : I \rightarrow A \otimes A\otimes A^* \otimes A^* $\\
\input{tikz/cap_cpmc.tikz}&\input{tikz/cap_c.tikz}\\
\multicolumn{2}{c}{$\eta: 1 \mapsto \sum_{ij}\ket{e_i} \otimes \ket{e_j} \otimes \ket{e_i} \otimes \ket{e_j} $}\\
&\\
\input{tikz/tensor_cpmc.tikz} &\input{tikz/tensor_c.tikz}\\
\multicolumn{2}{c}{$f_1 \otimes f_2:A^* \otimes C^*\otimes C\otimes A \rightarrow B^* \otimes D^* \otimes D \otimes B $}\\
  \hline
\end{tabular}
\end{table}

\begin{table}[h!]
\centering
\caption{Table of diagrams for Frobenius algebras in $\CPMC$ and $\mathcal{C}$}
\label{tab:frob} 
\begin{tabular}{ c c}
  $\CPMC$ & $\mathcal{C}$\\
  \hline
    $E(\mu) = \mu_* \otimes \mu$ & $\mu : A^* \otimes A \otimes A^* \otimes A \rightarrow A^* \otimes A$ \\
  \input{tikz/mu_cpmc.tikz}&\input{tikz/mu_c.tikz}\\
  \multicolumn{2}{c}{$\mu: \ket{e_i} \otimes \ket{e_j} \otimes \ket{e_k} \otimes \ket{e_l} \mapsto \braket{e_i| e_k}\braket{ e_j| e_l}(\ket{e_i} \otimes \ket{e_j})$} \\
$E(\Delta) = \Delta_* \otimes \Delta$ &  $\Delta : A^* \otimes A  \rightarrow A^* \otimes A \otimes A^* \otimes A $\\
\input{tikz/Delta_cpmc.tikz}&\input{tikz/Delta_c.tikz}\\
\multicolumn{2}{c}{$\Delta: \ket{e_i} \otimes \ket{e_j} \mapsto \sum_{ij}\ket{e_i} \otimes \ket{e_j} \otimes \ket{e_i} \otimes \ket{e_j} $} \\
$E(\iota) = \iota_* \otimes \iota$ & $\iota :  A^* \otimes A \rightarrow I$ \\
  \input{tikz/iota_cpmc.tikz}&\input{tikz/iota_c.tikz}\\
  \multicolumn{2}{c}{$\iota: \ket{e_i} \otimes \ket{e_j} \mapsto 1$} \\
$E(\zeta) = \zeta_* \otimes \zeta$ &  $\zeta : I \otimes A^* \otimes A  $\\ 
\input{tikz/xi_cpmc.tikz}&\input{tikz/xi_c.tikz}\\
\multicolumn{2}{c}{$\zeta: 1 \mapsto \sum_{i}\ket{e_1} \otimes \ket{e_i}$} \\ 
  \hline
\end{tabular}
\end{table}

\subsection{Pregroup grammar}
We use Lambek's pregroup grammar \cite{La99}.  A pregroup   $(P, \leq, \cdot, 1, (-)^l, (-)^r)$ is a partially ordered monoid $(P, \leq, \cdot, 1)$ where each element $p\in P$ has a left adjoint $p^l$ and a right adjoint $p^r$, such that the following inequalities hold:
\begin{equation}
  \label{eq:preg}
  p^l\cdot p \leq 1 \leq p\cdot p^l \quad \text{ and } \quad p\cdot p^r \leq 1 \leq p^r \cdot p
\end{equation}
Intuitively, we think of the elements of a pregroup as linguistic types. The monoidal structure allows us to form composite types,
and the partial order encodes type reduction. The important right and left adjoints then enable the introduction of types requiring
further elements on either their left or right respectively.

We understand a pregroup as a compact closed category in the following way. The objects of the category are the elements of the set $P$. The tensor and unit are the monoid multiplication and unit  $\cdot$, $1$, and cup and caps are the morphisms witnessed by the inequalities in \eqref{eq:preg}.

The pregroup grammar~$\freepreg{\mathcal{B}}$ over an alphabet~$\mathcal{B}$ is freely constructed from the atomic types in~$\mathcal{B}$. 
Here we use an alphabet $\mathcal{B} = \{n, s\}$, where we use the type $s$ to denote a declarative sentence and $n$ to denote a noun. A transitive verb can then be denoted $n^r s n^l$. If a string of words and their types reduces to the type $s$, the sentence is judged grammatical. The sentence \lang{dogs chase cars} is typed $n~(n^r s n^l)~ n$, and can be reduced to $s$ as follows: 
\[
n~(n^r s n^l)~ n \leq 1\cdot s n^l n \leq 1 \cdot s \cdot 1 \leq s
\]
This symbolic reduction can also be expressed graphically, as shown in figure~\ref{fig:reduction}.
\begin{figure}[htbp]
\centering
\input{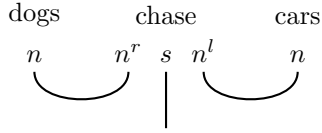}
\caption{A transitive sentence in the graphical calculus.}
\label{fig:reduction}
\end{figure}

\subsection{Mapping from grammar to meaning}
\label{sec:preg}

We now describe a functor from the grammar category $\freepreg{\{n, s\}}$ to $\CPM{\FHilb}$ that tells us how to compose word representations to form phrases and sentences.  The reductions of the pregroup grammar  may be mapped into $\CPM{\FHilb}$ using a strong monoidal functor $\cpmsem$:
\[
\cpmsem: \mathbf{Preg} \rightarrow \CPM{\FHilb}
\]
Strong monoidal functors automatically preserve the compact closed structure.
For our example~$\freepreg{\{n,s\}}$, we must map the noun and sentence types to appropriate finite dimensional vector spaces:
\[
\cpmsem(n) = N^* \otimes N \qquad \cpmsem(s) = S^* \otimes S
\]
where $N$, $S$ are finite dimensional Hilbert spaces, i.e.\ objects of $\FHilb$.
Composite types are then constructed functorially using the corresponding structure in $\FHilb$.
Each morphism $\alpha$ in the pregroup is mapped to a completely positive map interpreting sentences of that grammatical type. Since the only basic morphisms in the pregroup are identity, cups, and caps, $\alpha$ consists of tensor products and compositions of these. Then, given psd matrices for words $\sem{w_i}$ with pregroup types $p_i$, and  a type reduction in the pregroup grammar $\alpha: p_1, p_2, ... p_n \rightarrow s$, the meaning of the sentence $w_1 w_2 ... w_n$ is given by:
\[
\sem{w_1 w_2 ... w_n} = \cpmsem(\alpha)(\sem{w_1} \otimes \sem{w_2} \otimes ... \otimes \sem{w_n})
\]

\begin{exa}
  Let the space~$N$ be a real Hilbert space with basis vectors given by~$\{\ket{n_i}\}_i$, and suppose we have 
  \[
  \sem{\lang{cars}} = \sum_{ij} c_{ij}\ket{n_i}\bra{n_j}, \quad  \sem{\lang{dogs}} = \sum_{kl}  d_{kl}\ket{n_k}\bra{n_l}
  \]
 Let $S$ have basis $\{\ket{s_i}\}_i$. The verb $\sem{\lang{chase}}$ is given by:
\[
\sem{\lang{chase}} = \sum_{pqrtuv} C_{pqrtuv} \ket{n_p}\bra{n_t} \otimes \ket{s_q}\bra{s_u} \otimes \ket{n_r}\bra{n_v}
\]

The meaning of the composite sentence is~$(\epsilon_{N^* \otimes N} \otimes 1_{S^*\otimes S} \otimes \epsilon_{N^* \otimes N})$
applied to~$(\sem{\lang{dogs}} \otimes \sem{\lang{chase}} \otimes \sem{\lang{cars}})$ as shown below in \eqref{eq:ts},
with interpretation in~$\FHilb$ shown in \eqref{eq:tscpmc}.
\begin{equation}
\label{eq:ts}
\CPM{\FHilb}: \begin{gathered} \input{tikz/sentence_cpmc.tikz} \end{gathered}
\end{equation}
\begin{equation}
\label{eq:tscpmc}
\FHilb: \begin{gathered} \input{tikz/sentence_c_mixed.tikz} \end{gathered} \cong
\begin{gathered} \input{tikz/sentence_c_boxes.tikz} \end{gathered}
\end{equation}
\end{exa}
}

\end{document}